\documentclass{article}

\usepackage{microtype}
\usepackage{graphicx}
\usepackage{subcaption}
\usepackage{booktabs} % for professional tables
\usepackage{amsmath}
\usepackage{multirow}

\usepackage{hyperref}

\newcommand{\piold}{\pi_{\theta_\text{old}}}
\newcommand{\pinew}{\pi_{\theta}}

\newcommand{\grpo}{\text{GRPO}}
\newcommand{\grpoclip}{\grpo\text{-clip}}
\newcommand{\grponoclip}{\grpo\text{-noclip}}
\newcommand{\pppo}{P3O}
\newcommand{\TOPR}{\text{TOPR}}
\newcommand{\topr}{\text{TOPR}}
\newcommand{\SAPO}{\text{SAPO}}
\newcommand{\sapo}{\text{SAPO}}
\newcommand{\ourmethod}{\text{PSPO}}
\newcommand{\ourmethodgr}{\text{GR-\ourmethod}}

\usepackage{amsmath}
\usepackage{amssymb}
\usepackage{mathtools}
\usepackage{amsthm}
\usepackage{algorithm}
\usepackage{url}
\usepackage{xcolor}
\usepackage{algpseudocode}
\usepackage{natbib}
\usepackage{authblk}

\usepackage[capitalize,noabbrev]{cleveref}

\theoremstyle{plain}
\newtheorem{theorem}{Theorem}[section]
\newtheorem{proposition}[theorem]{Proposition}
\newtheorem{lemma}[theorem]{Lemma}
\newtheorem{corollary}[theorem]{Corollary}
\theoremstyle{definition}

\theoremstyle{remark}

\usepackage[textsize=tiny]{todonotes}

\title{It’s Not You, It’s Clipping: A Soft Trust-Region via Probability Smoothing for LLM RL}

\author[1]{Madeleine Dwyer}
\author[1,2]{Adam Sobey}
\author[1]{Adriane Chapman}

\affil[1]{University of Southampton}
\affil[2]{The Alan Turing Institute}

\date{January 2026}

\begin{document}

\maketitle




\begin{abstract}

Training large language models (LLMs) with reinforcement learning (RL) methods such as PPO and GRPO commonly relies on ratio clipping to stabilise updates. While effective at preventing instability, clipping discards information, introduces gradient discontinuities and can prevent exploration of better policies. Inspired by label smoothing, we propose Probability Smoothing Policy Optimisation (PSPO). PSPO smooths current policy probabilities toward the behaviour policy before computing importance ratios, creating a soft trust region that preserves gradients while preventing destabilising updates. Unlike prior soft clipping approaches that use sigmoid-based transformations which can suffer from vanishing gradients and saturation, our method uses a linear interpolation, providing simpler and more robust gradient preservation.

Empirically, GR-PSPO outperforms clipping and sigmoid-based alternatives on mathematical reasoning benchmarks when refining models with prior domain knowledge, achieving an accuracy of 79.9\% on GSM8K and 59.6\% on MATH for Qwen2-Math-1.5B.

\end{abstract}

\section{Introduction}

Reinforcement learning (RL) is now a central component of large language model (LLM) fine-tuning. RL is commonly used after supervised fine-tuning (SFT)~\citep{ouyang_training_2022}, where the aim is to refine an existing policy whilst preserving prior knowledge. In this setting, the main challenge is to make improvements without degrading existing knowledge, as unstable, destructive updates can waste expensive training.
RL post-training has shown enhanced reasoning and can maintain context in long and complex conversations. This is especially important with the prevalence of Small Language Models, that due to their size are energy efficient and easier to control in a workflow as they can be trained and maintained by smaller companies but where these models in particular struggle to provide outputs that are understandable by a human user. 

%Proximal Policy Optimization (PPO;~\citealp{schulman_proximal_2017}) while Group Relative Policy Optimization (GRPO) adapts PPO specifically  for LLMs a

To stabilise updates, policy optimisation methods rely on constraints that limit deviation from a reference policy. In practice, this requires a balance between learning speed and stability, as overly conservative updates can make learning inefficient or infeasible. Trust Region Policy Optimisation (TRPO)~\citep{schulman_trust_2017} constrains updates using the KL divergence, which allows for larger steps but is still computationally inefficient. Proximal Policy Optimization (PPO)~\citep{schulman_proximal_2017} provides an empirically stronger regularisation by using clipped probability ratios as a first-order approximation of the KL divergence. PPO clips the importance sampling ratio to effectively reduce the step size when the policy is outside the trust region, and underpins systems such as WebGPT~\citep{nakano_webgpt_2021}, LLaMA-2 Chat~\citep{touvron_llama_2023}, and Sparrow~\citep{glaese_improving_nodate}. 

This is extended in Group Relative Policy Optimization (GRPO) which is an adaptation of PPO specifically for LLMs ~\citep{shao_deepseekmath_2024}, inheriting the clipped probability, and has been applied to mathematical reasoning tasks~\citep{shao_deepseekmath_2024}, alongside other RL approaches~\citep{luong_reft_2024,mitra_motif_2025,luo_wizardmath_2025,zheng_group_2025}. Some implementations of GRPO~\citep{trl_grpo_docs} default to using a single pass over data, effectively avoiding clipping by only using on-policy data. This leads to the importance sampling ratio always being 1, which essentially reverts the approach back to a vanilla policy gradient method, and as such, will typically require small steps and be sample inefficient.

However, ratio clipping has drawbacks, namely vanishing gradients when the policy ratio leaves the clip range. Additionally, clipping can miss better policies outside of the clipped policy space~\citep{chen_sufficiency_2023}, especially in problems where greater exploration might be beneficial. Alternatives to clipping have therefore emerged, predominantly focusing on smooth transforms to remove the brittleness of hard clipping. P3O~\citep{chen_sufficiency_2023} introduced sigmoid-based soft clipping through its SCOPIC objective to enable exploration beyond the trust region defined by typical hard clip boundaries. P3O was empirically shown to perform better than PPO on five deep RL gym environments, but has not been applied to GRPO or used for LLM fine-tuning prior to this work. Sigmoid methods can, however, lead to saturation of values at extreme ratios, reintroducing vanishing gradients in precisely the regions where clipping also fails. Recent work has addressed the asymmetric behaviour of positive versus negative advantages in LLM fine-tuning specifically, Tapered Off-Policy REINFORCE (TOPR)~\citep{roux_tapered_2025} employs asymmetric importance sampling to differentially weight positive and negative advantages showing better performance on GSM8K compared with PPO, DPO and REINFORCE on 8B parameter models. Soft Adaptive Policy Optimization (SAPO)~\citep{gao_soft_2025} combines sigmoid soft gating with asymmetric control, similar to ideas from both P3O and TOPR, showing an increase in performance compared with GRPO and GSPO on mathematical tasks on 30B parameter models. Asymmetric methods require at least one additional parameter to be tuned, and TOPR's performance seems to be quite dependent on the ratio of positive to negative samples in the training, which is not always known a priori or easy to control especially given that it will change across training. Other approaches include KL early stopping~\citep{sun_you_2022}, though these can be brittle or saturating, particularly in complex settings.

We propose \emph{Probability Smoothing Policy Optimisation} (\ourmethod{}), as an alternative to clipping. Instead of truncating ratios, we smooth the current policy’s probabilities toward the old behaviour policy before computing the importance ratio. This is inspired by label smoothing in supervised learning. This smoothing reduces overconfidence in any single action while retaining informative gradients everywhere. 

We empirically evaluate \ourmethod{} by instantiating it within GRPO~\citep{shao_deepseekmath_2024} (\ourmethodgr{}) across three model types: Qwen2.5-0.5B, Qwen2.5-1.5B, and Qwen2-Math-1.5B. This experimental design allows us to test our hypothesis that \ourmethod{}'s advantages emerge when stability is required, such as with models with prior domain knowledge (the Math model). 

%Results validate this prediction: on the model already pretrained with mathematics, \ourmethodgr{} wins on all four benchmarks (GSM8K~\citep{cobbe_training_2021}, ASDiv~\citep{miao_diverse_2020}, SVAMP~\citep{patel_are_2021}, MATH~\citep{lightman_lets_2023}), while on base models the advantages are most pronounced on the hardest tasks. These findings demonstrate that probability smoothing excels at refining models with prior base knowledge you want to retain.

\section{Probability Smoothing Policy Optimisation}

Policy gradient methods optimise the expected reward by updating the policy $\pinew$ with respect to sampled trajectories. To effectively reuse trajectories from an old policy $\piold$, any update is regularised using importance sampling. Importance sampling estimates how likely the (state $s_t$, action $a_t$) pair would occur given the current policy. In PPO~\citep{schulman_proximal_2017} and GRPO~\cite{shao_deepseekmath_2024}, this is approximated with a ratio of the current policy $\pinew$ and the old, behaviour policy $\piold$ which generated the trajectory. This ratio is defined in equation ~\ref{eq:ratio},
\begin{equation}
r_t(\theta) = \frac{\pinew(a_t \mid s_t)}{\piold(a_t \mid s_t)}.
\label{eq:ratio}
\end{equation}

\textbf{GRPO} is an adaptation of PPO, ~\citep{shao_deepseekmath_2024}, which removes the need for a critic model, reducing the amount of training resources and developed specifically for LLMs. GRPO samples a group $G$ of outputs $a$ for a given prompt $s \in S$, and uses the group scores $r$ as a baseline estimate to then calculate the advantage $\hat{A}$ using the relative rewards based on the current group baseline; $\hat{A}_{t,i} = R_{t,i} - \Bar{R}_{t,i}$. GRPO includes the same clipping principle as PPO in its surrogate objective, although some default implementations suggest that using GRPO with only 1 iteration over the data gives comparable performance and negates the effect of clipping. GRPO aims to maximise the objective function:

\begin{equation}
\begin{split}
J^{\text{GRPO}}(\theta) = 
\mathbb{E}_t \Bigg[
\frac{1}{G} \sum_{i=1}^{G} \Bigg\{
\min \Big( r_{t,i}(\theta) \hat{A}_{t,i}, \, \text{clip}(r_{t,i}(\theta), 1-\varepsilon, 1+\varepsilon) \hat{A}_{t,i} \Big) - \beta\mathbb{D}_{KL}[\pinew \mid \pi_{ref}] \Bigg\}\Bigg],
\label{eq:grpo}
\end{split}
\end{equation}

where $\mathbb{D}_{KL}[\pinew \mid \pi_{ref}]$ is an estimate of the KL divergence from the current policy $\pinew$ to a reference policy $\pi_{ref}$, and $\beta$ is a hyper-parameter which controls the strength of this penalty. This KL divergence is similar to that used by TRPO, but in GRPO it is used as a soft penalty of $\pinew$ to $\pi_{ref}$ compared with TRPO's hard constraint of $\piold$ to $\pinew$. In some popular implementations~\citep{trl_grpo_docs}, $\beta$ is set to 0 as it reduces memory usage and improves the training speed by not needing to load the reference model.

In complex RL problems, there is often multiple optimal actions. Language generation tasks demonstrate this excellently, as within language, there are typically many possible words (actions) that can represent the same meaning (achieve the same goal).

To reduce overconfidence in any single action in a given state, we took inspiration from the label smoothing regularisation method used in supervised learning~\citep{szegedy_rethinking_2016}. Label smoothing has been shown to reduce overconfidence and improve the robustness of a model~\citep{muller_when_nodate}~\citep{goibert_adversarial_2019}.
Label smoothing, equation ~\ref{eq:label_smoothing}, moves from one-hot encoded target distribution $\varphi(k \mid x)$ to soft targets $\tilde{\varphi}(k \mid x)$ that are a weighted average of the hard target distribution and another distribution, traditionally the uniform distribution $u(k)$~\citep{szegedy_rethinking_2016}, \footnote{In the original label smoothing paper, $\varepsilon$ is used as the smoothing parameter, we use $\alpha$ to avoid confusion with the clipping range, often denoted as $\varepsilon$.}

\begin{equation}
\tilde{\varphi}(k \mid x) = (1 - \alpha) \cdot \varphi(k \mid x) + \alpha \cdot u(k),
\label{eq:label_smoothing}
\end{equation}
where $\alpha \in [0, 1]$ controls the smoothing strength.

We apply \eqref{eq:label_smoothing} to the current policy probability in equation ~\ref{eq:smoothed_prob_step},
\begin{equation}
\tilde{\pinew}(a_t \mid s_t) = (1-\alpha)\pinew(a_t \mid s_t) + \alpha \cdot q(a_t \mid s_t),
\label{eq:smoothed_prob_step}
\end{equation}
where $q(\cdot)$ represents the distribution we want to smooth towards. 

For policy optimisation, updates should be within a trust region to enable larger step updates. Therefore, we decided to smooth towards the old behaviour policy, $q = \piold$, rather than the uniform distribution, so the smoothing behaves as a behaviour-anchored trust region. ~\citet{szegedy_rethinking_2016} noted that the change in the loss when using label smoothing could be equivalently considered as an estimate of the KL divergence. Given that PPO used the probability ratio clipping as a first-order estimate of KL-divergence for regularisation, this bolsters our intuition to smooth towards the old policy, and by doing so, we introduce an equivalent estimate of the KL-divergence with the smoothed action probability and create a soft trust region. The smoothed probability becomes equation ~\ref{eq:smoothed_prob},
\begin{equation}
\tilde{\pinew}(a_t \mid s_t) = (1-\alpha)\pinew(a_t \mid s_t) + \alpha \cdot \piold(a_t \mid s_t).
\label{eq:smoothed_prob}
\end{equation}

If we then consider the ratio equation, we can find the effect from the smoothed probability on the ratio using equation ~\ref{eq:smoothed_prob},
\begin{equation}
    \tilde{r}_t(\theta) = \frac{\tilde{\pinew}(a_t \mid s_t)}{{\piold}(a_t \mid s_t)},
\label{eq:smoothed_ratio}
\end{equation}
which given \eqref{eq:smoothed_prob}, becomes equation~\ref{eq:smoothed_ratio_full},

\begin{equation}
\begin{split}
\tilde{r}_t(\theta) = \frac{(1-\alpha)\pinew(a_t \mid s_t) + \alpha \cdot \piold(a_t \mid s_t)}{\piold(a_t \mid s_t)} \\ = (1-\alpha)r_t + \alpha.
\label{eq:smoothed_ratio_full}
\end{split}
\end{equation}

Having derived the smoothed ratio (Eq.~\ref{eq:smoothed_ratio_full}) we now establish formally that this simple linear transformation provides the properties we intuitively assumed: (i) contraction toward the behaviour policy in both total variation and KL divergence, ensuring updates remain within a trust region; (ii) preservation of non-zero gradients everywhere, avoiding the learning signal loss caused by clipping; and (iii) regularisation against overconfidence in any single action. We state these results as formal propositions.

\paragraph{Soft Trust Region - Implicit divergence control from probability smoothing.}
Given the smoothed policy and ratio (Eqs.~(~\ref{eq:smoothed_prob_step})–(~\ref{eq:smoothed_ratio_full})), the linear interpolation, $\tilde{r}_t(\theta) = (1-\alpha)r_t + \alpha$, yields a contraction around $r=1$ and induces a soft trust region anchored at $\piold$, consistent with our intuition.

\begin{lemma}[Total variation (TV) contraction]
\label{lem:tv}
For any state $s$ and $\alpha\in[0,1]$,
\[
\|\tilde\pinew(\cdot\mid s)-\piold(\cdot\mid s)\|_{1}
=(1-\alpha)\,\|\pinew(\cdot\mid s)-\piold(\cdot\mid s)\|_{1}.
\]
\end{lemma}
\begin{proof}
Since $\tilde\pinew-\piold=(1-\alpha)(\pinew-\piold)$ pointwise, linearity of the $\ell_1$ norm gives the result directly.
\end{proof}

This means $\alpha$ directly controls the maximum distance the smoothed policy can drift from the behaviour policy: setting $\alpha = 0.1$ guarantees the smoothed policy is at most 90\% as far from $\piold$ as the unsmoothed policy would be, in total variation distance. This provides a tunable constraint on policy updates without hard truncation.

\begin{corollary}[KL upper bounds shrink under smoothing]
\label{cor:kl}
We use the joint convexity of KL and
set $\lambda = 1 - \alpha$, $P_1 = \pinew$, $P_2 = \piold$, $Q_1 = \piold$, $Q_2 = \piold$. This gives us:
\begin{equation*}
\begin{split}
\lambda P_1 + (1 - \lambda)P_2 = (1-\alpha)\pinew + \alpha\piold = \tilde{\pinew},
\\
\lambda Q_1 + (1 - \lambda)Q_2 = (1 - \alpha)\piold + \alpha\piold = \piold
\end{split}
\end{equation*}
Given that $D_{\mathrm{KL}}\!\big(\piold\|\ \piold\big) = 0$, we then find:
\[
D_{\mathrm{KL}}\!\big(\tilde\pinew\,\|\,\piold\big)
\;\le\;
(1-\alpha)\,D_{\mathrm{KL}}\!\big(\pinew\,\|\,\piold\big)
\]
Similarly for the reverse direction we find:
\[
D_{\mathrm{KL}}\!\big(\piold\,\|\,\tilde\pinew\big)
\;\le\;
(1-\alpha)\,D_{\mathrm{KL}}\!\big(\piold\,\|\,\pinew\big).
\]
Hence $\alpha$ directly sets a \emph{soft trust-region} radius in both total variation (TV) and (upper-bounded) KL.
\end{corollary}

This establishes that probability smoothing induces an implicit KL penalty, which can be viewed as a soft trust-region with radius of $\alpha$.

\begin{proposition}
[Ratio contraction and non-vanishing slopes]
\label{prop:ratio}
For any action $a$ with $\piold(a\mid s)>0$, and $r(a)$ is the importance sampling ratio for action $a$,
\[
|\tilde r(a)-1|\le(1-\alpha)\,|r(a)-1|,
\qquad
\frac{\partial}{\partial r}\big(\tilde r\,A\big)=(1-\alpha)A.
\]
Thus, \ourmethod{} preserves slope $(1-\alpha)A$ everywhere, avoiding the flat plateaus introduced by clipping outside $[1-\varepsilon,1+\varepsilon]$ (Fig.~\ref{fig:ratio}).
\end{proposition}

This is the key advantage over clipping. With PPO/GRPO clipping, when $r > 1 + \varepsilon$ (for positive advantages) or $r < 1 - \varepsilon$ (for negative advantages), the gradient is exactly zero, the model receives no learning signal for these samples. PSPO instead scales the gradient by $(1 - \alpha)$, which is non-zero everywhere. This means PSPO continues to learn from samples that clipping would ignore, while the contraction still prevents destabilising large updates.

\begin{proposition}[Overconfidence regularisation]
\label{prop:overconfidence}
For any state $s$ and action $a$, the smoothed policy satisfies:
\[
\tilde{\pinew}(a \mid s) \leq max\Big(\pinew(a \mid s), \piold(a \mid s)\Big),
\] with strict inequality whenever $\pinew(a \mid s) \neq \piold(a \mid s)$ and $\pinew(a \mid s) > \piold(a \mid s)$.
\begin{proof}
    From the definition $\tilde{\pinew}(a \mid s) = (1 - \alpha)\pinew(a \mid s) + \alpha\piold(a \mid s)$. When $\pinew(a \mid s) \geq \piold(a \mid s)$, we have $\tilde{\pinew}(a \mid s) < \pinew(a \mid s)$. When $\pinew(a \mid s) < \piold(a \mid s)$, we have $\tilde{\pinew}(a \mid s) < \piold(a \mid s).$ 
\end{proof}
\end{proposition}

This property prevents the policy from over fitting and becoming overly confident in any single action. When the current policy assigns higher probability to an action than the behaviour policy did, smoothing pulls this probability down. This is particularly relevant for language generation, where multiple tokens may be equally valid continuations.

\begin{proposition}[\ourmethod{} surrogate as a scaled policy gradient with implicit stability]
\label{prop:pgradient}
The per-state \ourmethod{} objective can be written
\begin{equation*}
\begin{split}
\mathcal{J}_{\ourmethod}(\theta)
=\mathbb{E}_{a\sim\piold}\![\tilde r(a)\,A(a)] \\
=(1-\alpha)\,\mathbb{E}_{a\sim\pinew}\![A(a)]
+\alpha\,\mathbb{E}_{a\sim\piold}\![A(a)].
\end{split}
\end{equation*}
using equation~\ref{eq:ratio} and the change of measure formula.
Only the first term depends on $\theta$, and using policy gradient theorem,
\(
\nabla_\theta \mathcal{J}_{\ourmethod{}}
=(1-\alpha)\,\mathbb{E}_{a\sim\pinew}\!\big[\nabla_\theta\log\pi_\theta(a\mid s)\,A(a)\big].
\)
Hence, \ourmethod{} is the \emph{on-policy} gradient scaled by $(1 - \alpha)$, \emph{while} the policy itself is mixed with $\piold$ (Lemma~\ref{lem:tv}), implicitly controlling divergence without an explicit KL term (which we set $\beta{=}0$ in our GRPO runs).
\end{proposition}

This shows that PSPO can be seen as performing standard on-policy gradient descent, but with two modifications: the gradient is scaled by $(1 - \alpha)$, slowing learning, while the policy used to compute advantages is mixed with $\piold$. Together, these effects stabilise learning by anchoring updates to recent behaviour.

Proposition \ref{prop:pgradient} shows that PSPO scales the policy gradient by a factor of $1 - \alpha$. While this may resemble a simple reduction in the learning rate, PSPO is not equivalent to learning-rate scaling alone. The scaling arises from probability mixing with the behaviour policy, which simultaneously contracts the policy distribution toward $\piold{}$. As a result, PSPO adaptively alters the update magnitudes based on the model’s own prior probabilities, rather than applying a uniform step-size reduction. In contrast, clipping also implicitly rescales gradients, but does so discontinuously by truncating updates past a fixed ratio threshold.

Figure~\ref{fig:ratio} illustrates how \ourmethod{} and clipping affect the ratio and the gradients. Clipping creates flat regions where gradients vanish: for $A>0$, the clipped surrogate is constant for $r>1+\varepsilon$; for $A<0$, it is constant for $r<1-\varepsilon$. In contrast, our method smooths the current policy toward the behaviour policy, giving us Eq.~\ref{eq:smoothed_ratio_full}. This smooths ratios toward $1$, creating a soft trust region anchored by $\piold$, while maintaining non-zero gradients everywhere: $\tfrac{\partial}{\partial r}(\tilde r A)=(1-\alpha)A$. Therefore, \ourmethod{} preserves learning signal outside the clip range whilst still controlling updates.

\begin{figure}
    \centering
    \includegraphics[width=0.95\linewidth]{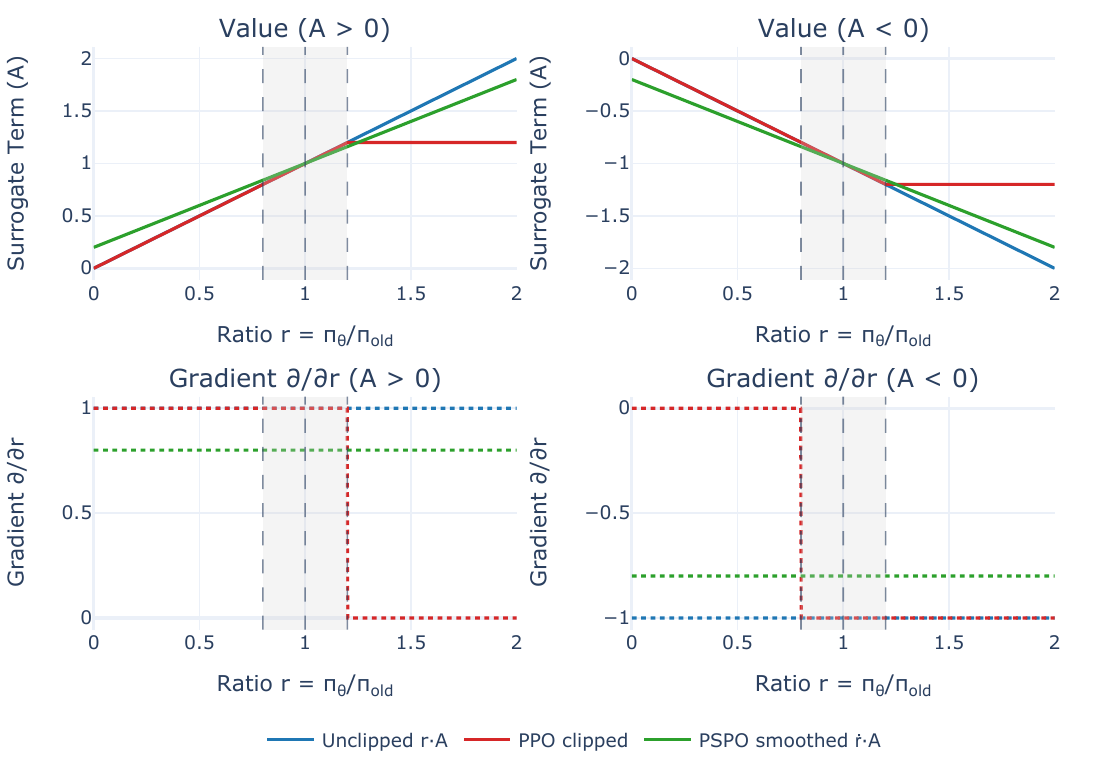}
    \caption{Illustrative plot of ratio $r$ vs. the surrogate term $A$, and the gradients for $A>0$ and $A<0$, with $\varepsilon = 0.2$ and $\alpha = 0.2$. For $A > 0$ the clipped ratio is flat (zero gradient) for $r > 1 + \varepsilon$; for $A < 0$, the clipped ratio is flat when $r < 1 - \varepsilon$. PSPO's slope is $(1-\alpha)A$ everywhere, creating a soft trust region without hard plateaus.}
    \label{fig:ratio}
\end{figure}

\paragraph{Applicability.}
\ourmethod{} is a direct replacement for ratio clipping, requiring only the substitution $\tilde r_t=(1-\alpha)r_t+\alpha$ for $r_t$ in any clipped-ratio objective. This change requires no additional computation or memory beyond evaluating the usual importance ratio.

\paragraph{Selecting $\alpha$.} The smoothing parameter $\alpha \in [0, 1]$ controls the strength of the trust region by determining the fraction of the old policy to mix with the current policy. 
%$\alpha = 0$ is no smoothing, so the full, unmodified, importance sampling ratio will be used. and $\alpha = 1$ is fully smoothed, so the importance sampling ratio will always be 1. 
The larger the value of $\alpha$, the stronger the regularisation is, and the ratio is contracted more toward 1, and the KL upper bounds shrink (Corollary \ref{cor:kl}). A larger $\alpha$ is analogous to a smaller clipping range $\varepsilon$ in traditional GRPO.
In practice, we have found that $\alpha \in [0.1, 0.4]$ works well across our experiments, comparable to typical clipping ranges used in PPO/GRPO. Importantly, $\alpha$ replaces $\varepsilon$ as the trust region hyperparameter: both methods require tuning exactly one scalar controlling the constraint strength, so PSPO introduces no additional hyperparameter burden.

\paragraph{Application to GRPO:} We demonstrate how \ourmethod{} can apply to GRPO, denoted as \ourmethodgr{} which changes \eqref{eq:grpo} to equation ~\ref{eq:gr-pspo},
 \begin{equation}
J^{\text{GR-PSPO}}(\theta) = 
\mathbb{E}_t \Bigg[\frac{1}{G} \sum_{i=1}^{G}
( \tilde{r}_{t,i}(\theta) \hat{A}_{t,i}) - \beta\cdot\mathbb{D}_{KL}[\pinew \mid \pi_{ref}]\Bigg].
\label{eq:gr-pspo}
\end{equation}

\section{Experimental Setup}
% \subsection{Shared Info}
\textbf{Model and Prompt Formatting}
We fine-tune the open source causal LMs Qwen2.5-0.5B and -1.5B~\citep{qwen2.5} and also the Qwen2-Math-1.5B model which has specialised domain knowledge of mathematics~\citep{yang2024qwen2} using their own tokeniser. The primary difference between the Math and the base models is that the Math model is a Qwen2-1.5B model which is then pretrained on a mathematics-specific corpus using supervised fine-tuning.

We use these three models as they are openly available, allow us to test scale across different, smaller sizes and enables us to compare models with and without existing mathematics knowledge. All runs use identical tokenisation and prompt formatting. Each sample is formatted with a \emph{system} instruction followed by the \emph{user} problem text. We use the model’s native chat template via \texttt{tokenizer.apply\_chat\_template(\dots, add\_generation\_prompt{=}True)} to append the assistant header and ensure the model completes in the assistant role, consistent with how the models are trained to be used.

When decoding completions, we set \texttt{max\_completion\_length=128} with a maximum prompt length of 512 tokens, giving 640 tokens overall. Token count in LLM fine-tuning is a big contributing factor on the compute required (approx. 4 times the token count memory required for training), so we chose 128 completion tokens to balance computational constraints whilst allowing room for the response to include reasoning steps. This is also the approximate length of the GSM8K gold answers. We do not enforce any additional stop strings beyond the template defined end of string token.

\textbf{Training}
All methods are trained using the same effective batch size and decoding settings. We train each method across 3 seeds, for the same number of tokens for fair comparison of the different methods; the best-validation checkpoint model is used for evaluations.

For the base models (Qwen2.5-0.5B and -1.5B), which lack specialised knowledge, we trained for 150M tokens. For Qwen2-Math-1.5B, which has existing mathematical knowledge to build on, we train for 75M tokens. Models with existing domain knowledge typically require less training to refine.

\textbf{Methods}
We compare \ourmethodgr{} against several baselines using the HuggingFace Transformer Reinforcement Learning (TRL) implementation of GRPO~\citep{vonwerra2022trl}. \textbf{GRPO baselines:} \grpoclip{} (standard TRL with ratio clipping), \grponoclip{} (on-policy only: iterations=1, ensuring ratio$\equiv$1).
%and GRPO-noclip (clipping disabled, allowing full off-policy IS ratios across multiple iterations). 
\textbf{Alternative trust regions:} We benchmark against the current alternatives to hard clipping; \pppo's SCOPIC objective~\citep{chen_sufficiency_2023} and \TOPR's asymmetric gating~\citep{roux_tapered_2025}, originally designed for PPO and REINFORCE respectively, into the GRPO framework. We also compare against \SAPO~\citep{gao_soft_2025}, which integrates \pppo's smooth sigmoidal scaling with \TOPR's asymmetric advantage weighting, using their TRL implementation backported to the TRL version we use. Table~\ref{tab:methods} summarises these methods and the naming conventions we will use throughout.
% We compare \ourmethod{} and \ourmethodfull{} against several baselines using HuggingFace TRL's GRPO implementation~\citep{trl}. \textbf{GRPO baselines:} GRPO-clip (standard TRL with ratio clipping), GRPO-noclip (on-policy with iterations=1, ensuring ratio≡1), and GRPO-noclip (clipping disabled, multiple iterations). \textbf{Alternative trust regions:} We also implemented \pppo's SCOPIC objective~\citep{\pppo} and \TOPR's asymmetric gating~\citep{\TOPR}—originally designed for PPO and REINFORCE respectively—into the GRPO framework. We also compare against \SAPO \citep{\SAPO}, which integrates \pppo's smooth sigmoidal scaling with \TOPR's asymmetric advantage weighting, using their TRL implementation backported to our version. Table~\ref{tab:methods} summarises the method and also how we refer to it.

\begin{table}[ht]
\centering
\caption{Trust region methods evaluated. All use GRPO's group relative advantages.}
\label{tab:methods}
\small
\begin{tabular}{p{0.2\columnwidth}p{0.35\columnwidth}p{0.3\columnwidth}}
Method & Trust Region Mechanism & Implementation \\
\midrule
\grpoclip & IS ratio clipping & TRL baseline \\
\grponoclip & On-policy (iterations=1 $\rightarrow$ ratio$\equiv$1) & TRL baseline \\
% \grponoclip & Full IS ratio, no clipping & Modified TRL \\
\midrule
\pppo & Sigmoidal soft clipping (SCOPIC) & Adapted to GRPO \\
\TOPR & Asymmetric advantage gating & Adapted to GRPO \\
\SAPO & Asymmetric sigmoidal scaling & Backported from TRL \\
\midrule
\textbf{\ourmethodgr{}} & \textbf{Probability smoothing} & \textbf{Novel (ours)} \\
\bottomrule
\end{tabular}
\end{table}

\textbf{Hyperparameter Tuning}
To tune the hyperparameters, we used a grid-search across learning rate, clipping range, smoothing alpha strength, and tau values. Memory-related and decoding parameters were fixed across all methods. More details on the tuning method and the parameters used for each method can be found in Appendix \ref{app:tuning}.

\textbf{Evaluation}
We evaluate each test set across temperatures $T\in\{0.0,0.2,0.4,0.6,0.8\}$ with top-$p{=}1.0$.
We report the mean zero-shot Top-1 accuracy with 95\% confidence intervals. We use Top-1 at $T=0$ as it reflects single-answer deployment settings and is deterministic. We calculate the accuracy using the true accuracy of a response, i.e. if anywhere in the response the correct answer appears. We also evaluate the reward function style accuracy as well.% We assess response quality using an LLM-as-Judge ~\citep{zheng2023judgingllmasajudgemtbenchchatbot}.

\subsection{Mathematical Reasoning Datasets}

% The initial experiments were trained on GSM8K (standard train/test split)~\citep{cobbe_training_2021} and evaluate in-domain on GSM8K and out-of-distribution on ASDiv~\citep{miao_diverse_2020}, SVAMP~\citep{patel_are_2021}, and MATH~\citep{lightman_lets_2023}. 
The base models were fine-tuned on GSM8K~\citep{cobbe_training_2021} and MATH~\citep{hendrycks_measuring_2021} (using the standard train/test split) and evaluated in-domain on the test sets and out-of-distribution on ASDiv~\citep{miao_diverse_2020} and SVAMP~\citep{patel_are_2021}. 
These benchmarks span basic arithmetic word problems (ASDiv), robustness to linguistic perturbations (SVAMP), and competition-level reasoning (MATH; sampled from MATH~\citep{hendrycks_measuring_2021}). Following Minerva~\citep{lewkowycz_solving_2022} and OpenWebMath~\citep{paster_openwebmath_2024}, we restrict evaluation to problems with numeric final answers to enable automatic verification. For GSM8K training, we split the published training set into 7000 train and 472 validation examples.
 % experiments were then trained on both GSM8K and MATH training datasets, matching common training methods from the literature~\citep{luo_wizardmath_2025}. These were tested in-domain on the GSM8K and MATH test splits, and out-of-distribution on the ASDiv and SVAMP problems.

\subsubsection{System Prompt}
\textbf{System:}
\begin{quote}
\small
You are a careful math solver. Think through the solution and show the steps. Use English only. End the response with the final answer only in the format: '\#\#\#\# $<$final numeric answer only$>$'.
\end{quote}
\textbf{User content:} the raw problem text (with no few-shot exemplars).

We use this prompting to encourage stepwise reasoning and finish with a single, numeric final answer, which can be more easily extracted for calculating the reward.
\subsubsection{Reward Function}
Our rewards follow the commonly used correctness-based setup~\citep{lewkowycz_solving_2022,paster_openwebmath_2024,deepseek-ai_deepseek-r1_2025}: $R{=}1$ for exact numeric correctness within $10^{-6}$ tolerance, plus a $+0.05$ shaping bonus if the output matches the format ``\#\#\#\# $<$number$>$''; values are constrained to $[0,1]$ giving $\{0,0.05,1\}$. We first attempt to extract the number from the requested format; if this is not present, we fall back to the last numeric token in the completion.

% \subsubsection{LLM-as-Judge Evaluation}
% We assess response quality by scoring 4 metrics on a scale of $1 - 5$ using an LLM-as-Judge (validating results by sampling a subset of the responses ourselves). The metrics scored are: constraint adherence (format fidelity, steps present); logical coherence (no contradictions, consistent rationale); mathematical soundness (valid operations/derivations); and clarity (concise, well-structured). Full details on the setup are in App.~\ref{app:llm-as-judge}.

% \subsection{Summarisation Dataset}
% To consider performance across tasks, we also trained models using the XSUM dataset for summarising, with inputs of news articles and one sentence summarisations. 

% \subsubsection{System Prompt}

% \textbf{System:}
% \begin{quote}
% \small
% Summarize the following article in one sentence. Use English only.
% \end{quote}
% \textbf{User content:} the news article to summarise.

% \subsubsection{Reward Function}
% We use the ROUGE scorer to determine rewards for the summaries generated. ....

% \subsubsection{LLM-as-Judge}
\begin{table*}[ht]
\centering
\small
\caption{Comparison of our results across all the model sizes on mathematical reasoning benchmarks. The accuracy presented is the true accuracy (i.e. has the correct answer anywhere in the response, even if it provides an incorrect answer at the end). Accuracy reported as average over 3 seeds, evaluated in greedy setting with T=0.0, and top\_p=1.0.}
\resizebox{\textwidth}{!}{%
\begin{tabular}{lcc|cccc}
\toprule
Model & Size & {Training} & \multicolumn{4}{c}{Accuracy (\%)} \\
& & Data & GSM8K & ASDiv & SVAMP & MATH \\
\midrule
Qwen2.5-0.5B       & 0.5B & N/A & 16.9 & 42.7 & 42.0 & 30.2 \\
Qwen2.5-0.5B+\grpoclip{}          & 0.5B & GSM8K & 54.1 & 77.8 & 66.6 & \textbf{43.0} \\
Qwen2.5-0.5B+\grponoclip{}        & 0.5B & GSM8K & \textbf{56.0} & 77.9 & \textbf{70.0} & 42.6 \\
Qwen2.5-0.5B+\pppo\cite{chen_sufficiency_2022} & 0.5B & GSM8K & 44.5 & 68.5 & 58.3 & 38.9 \\
Qwen2.5-0.5B+\sapo\cite{gao_soft_2025} & 0.5B & GSM8K & 31.8 & 62.0 & 51.8 & 36.2 \\
% Qwen2.5-0.5B+\topr\cite{roux_tapered_2025} & 0.5B & GSM8K &  \\
Qwen2.5-0.5B+\ourmethodgr{}   & 0.5B & GSM8K & 53.0 & \textbf{78.8} & 69.7 & 41.9 \\
\midrule
Qwen2.5-0.5B+\grpoclip{}          & 0.5B & GSM8K+MATH & 49.4 & 76.5 & 62.1 & 34.4\\
Qwen2.5-0.5B+\grponoclip{}        & 0.5B & GSM8K+MATH & \textbf{54.1} & \textbf{76.6} & 56.8 & 36.2\\
Qwen2.5-0.5B+\pppo\cite{chen_sufficiency_2022} & 0.5B & GSM8K+MATH & 50.1 & 75.0 & 62.8 & 36.1 \\
Qwen2.5-0.5B+\sapo\cite{gao_soft_2025} & 0.5B & GSM8K+MATH & 50.3 & \textbf{76.6} & 63.3 & 34.5 \\
Qwen2.5-0.5B+\topr\cite{roux_tapered_2025} & 0.5B & GSM8K+MATH & 2.8 & 3.2 & 7.3 & 8.6 \\
Qwen2.5-0.5B+\ourmethodgr{}   & 0.5B & GSM8K+MATH & 50.9 & 75.9 & \textbf{65.7} & \textbf{37.5} \\
\midrule
Qwen2.5-1.5B       & 1.5B & N/A     & 5.2 & 5.0 & 4.0 & 17.6 \\
% Qwen2.5-1.5B+\grpoclip{}          & 1.5B & GSM8K & 37.8 & 70.9 & 58.4 & 14.9  \\
% Qwen2.5-1.5B+\grponoclip{}        & 1.5B & GSM8K & 57.9 & \textbf{80.4} & \textbf{74.9} & 25.2 \\
% Qwen2.5-1.5B+\ourmethodgr{}   & 1.5B & GSM8K & \textbf{59.4} & 77.7 & 70.3 & \textbf{25.6} \\
% \midrule
Qwen2.5-1.5B+\grpoclip{}          & 1.5B & GSM8K+MATH & 73.7 & 88.8 & 83.3 & 45.9 \\
Qwen2.5-1.5B+\grponoclip{}        & 1.5B & GSM8K+MATH & \textbf{75.4} & 87.9 & \textbf{83.7} & 48.0 \\
Qwen2.5-1.5B+\ourmethodgr{}   & 1.5B & GSM8K+MATH & 73.8 & \textbf{89.4} & 82.7 & \textbf{50.0} \\
\midrule
Qwen2-Math-1.5B       & 1.5B & N/A     & 24.7 & 61.8 & 56.3 & 34.8 \\
% Qwen2-Math-1.5B+\grpoclip{}          & 1.5B & GSM8K+MATH &   \\
Qwen2-Math-1.5B+\grponoclip{}        & 1.5B & GSM8K+MATH & 78.4 & 92.8 & 84.6 & 58.3 \\
Qwen2-Math-1.5B+\grpoclip{}        & 1.5B & GSM8K+MATH & 70.3 & 91.6 & 82.0 & 52.3 \\
Qwen2-Math-1.5B+\pppo{}\cite{chen_sufficiency_2023}   & 1.5B & GSM8K+MATH & 78.8 & 93.0 & 84.7 & 55.8 \\
Qwen2-Math-1.5B+\sapo{}\cite{gao_soft_2025}   & 1.5B & GSM8K+MATH & 79.2 & 93.8 & 85.3 & 59.3 \\
Qwen2-Math-1.5B+\ourmethodgr{}   & 1.5B & GSM8K+MATH & \textbf{79.9} & \textbf{94.0} & \textbf{85.7} & \textbf{59.6} \\
% \midrule
% Qwen2-Math-7B       & 7B & -     & \\
\bottomrule
\end{tabular}
}
\label{tab:ourresults}
\end{table*}

\section{Performance of PSPO compared to the state-of-the-art}
We evaluate \ourmethodgr{} against multiple trust region baselines across mathematical reasoning tasks. Table \ref{tab:ourresults} shows results across all model configurations. Our main finding is that \ourmethodgr{}'s advantages grow with model size and domain specialisation, validating our hypothesis that probability smoothing is most effective when refining models that already have substantial knowledge, when stability is most important.

\subsection{Model Capability and Trust Region Performance}
 \ourmethodgr{}'s advantage scales with prior domain knowledge. On Qwen2.5-0.5B, differences between methods are small. On Qwen2.5-1.5B, \ourmethodgr{} achieves the strongest MATH performance (+4.1pp over \grpoclip{}). On Qwen2-Math-1.5B, a model with existing mathematical domain knowledge, \ourmethodgr{} outperforms on all four benchmarks.

This scaling is reflected in GSM8K performance, where GR-PSPO transitions from $\text{-}3.2$\% relative to the best per dataset (0.5B) to $\text{-}1.6$\% (1.5B) to +0.7\% (Math-1.5B).

\paragraph{Math-Specialised Model}
On Qwen2-Math-1.5B, \ourmethodgr{} achieves the strongest performance across all benchmarks. At $T=0$, \ourmethodgr{} attains 79.9\% on GSM8K (versus 78.4\% for GRPO-noclip), 94.0\% on ASDiv, 85.7\% on SVAMP, and 59.6\% on MATH.

The consistency across tasks—from arithmetic word problems (ASDiv, SVAMP) to competition-level reasoning (MATH)—suggests that probability smoothing generalises across difficulty levels when refining models with a base competence.

\subsection{Base Models}
On base Qwen2.5 models without prior mathematical training, \ourmethodgr{}'s advantages are most pronounced on MATH, the most challenging benchmark.
For Qwen2.5-0.5B trained on GSM8K+MATH, GR-PSPO achieves the highest MATH accuracy (37.5\% vs. 36.2\% for GRPO-noclip) but trails on GSM8K (50.9\% vs.\ 54.1\%). For Qwen2.5-1.5B, GR-PSPO achieves 50.0\% on MATH versus 45.9\% for GRPO-clip (+4.1pp), with comparable GSM8K performance.

When fine-tuning a base model on a new task, there is less existing knowledge to preserve, and aggressive updates may accelerate learning. When refining a model that already has domain expertise, stability becomes critical. GR-PSPO's behaviour-anchored smoothing provides this property, explaining its consistent gains on Qwen2-Math.

\subsection{Comparison to alternative methods.} 
GR-PSPO generally outperforms P3O and SAPO on both GSM8K and MATH across model sizes (Table 2). On Math-1.5B, GR-PSPO achieves 59.6\% on MATH versus 55.8\% for P3O (+3.8pp). %TOPR results were omitted due to [training instability / incomplete runs / etc.].
We also adapted TOPR~\citep{roux_tapered_2025} to the GRPO framework, but the model seems to collapse during training, leading to below base model performance. TOPR was originally designed for REINFORCE, so the model may of collapsed due to the group-relative advantage formulation in GRPO.

\subsection{Comparison to published results.} 
Our results outperform prior work at similar scales (Table \ref{tab:comparison}). \citet{luo_wizardmath_2025} report 50.1\% GSM8K and 21.2\% MATH on WizardMath-Large (0.7B); our 0.5B GR-PSPO achieves 50.9\% and 37.5\% respectively. Our Qwen2-Math-1.5B results (79.9\% GSM8K, 59.6\% MATH) substantially exceed WizardMath-GPT2-XL (58.9\%, 25.4\%) at comparable scale. 
More notably, our Qwen2-Math-1.5B results outperform the Llama 3 8B model from \citet{roux_tapered_2025}, achieving +4.5pp on GSM8K and +36.9pp on MATH despite being substantially smaller.
Direct comparison is complicated by differences in evaluation methodology; notably, our context window is much shorter (640 tokens with 128 completion tokens versus 2000+ tokens in prior work, or 128 completion tokens vs. the 512 completion tokens). We evaluate on the harshest settings, and still have competitive performance.

\begin{table*}[ht]
\centering
\small
\caption{Comparison on mathematical reasoning benchmarks. \textbf{Reward Source}: Rule = rule-based correctness function, Model = trained reward model. \textbf{Reward Type}: O = outcome (final answer), P = process (intermediate reasoning steps). The accuracy presented is the true accuracy (i.e. has the correct answer anywhere in the response, even if it provides an incorrect answer at the end). Note: ``—" indicates data could not be found; ``N/A" indicates not applicable.}
\resizebox{\textwidth}{!}{%
\begin{tabular}{lc|cc|cc|ccc|cc}
\toprule
\multirow{2}{*}{Model} & \multirow{2}{*}{Size} & \multicolumn{2}{c|}{Training} & \multicolumn{2}{c|}{Reward} & \multicolumn{3}{c|}{Evaluation} & \multicolumn{2}{c}{Accuracy (\%)} \\
\cmidrule(lr){3-4} \cmidrule(lr){5-6} \cmidrule(lr){7-9} \cmidrule(lr){10-11}
& & Data & Method & Source & Type & Prompting & Decoding & Tokens & GSM8K & MATH \\
\midrule
\multicolumn{11}{l}{\textbf{Smaller Language Models $<$ 1B parameters.}} \\
\midrule
WizardMath-Small \citep{luo_wizardmath_2025} & 0.1B & GSM8K+MATH & PPO & Model & P+O & - & T=0 & 2048 & 26.4 & 12.3 \\
WizardMath-Medium \citep{luo_wizardmath_2025} & 0.3B & GSM8K+MATH & PPO & Model & P+O & - & T=0 & 2048 & 38.7 & 15.6 \\
WizardMath-Large \citep{luo_wizardmath_2025} & 0.7B & GSM8K+MATH & PPO & Model & P+O & - & T=0 & 2048 & 50.1 & 21.2 \\
Qwen2.5-0.5B+RL \citep{zhuang_technical_2025} & 0.5B & GSM8K & GRPO & Rule & O & - & T=0.6, top\_p=0.95 & 32768 & 34.1 & 23.6 \\ %MATH
FlanT5-L Specialized \citep{fu_specializing_2023} & 0.76B & GSM8K & Distill & N/A & N/A & Few-shot & - & 8192 & 20.2 & - \\
GPT-2 (Socratic) \citep{shridhar_distilling_2023} & 0.77B & GSM8K+SCoT & Distill & N/A & N/A & Socratic CoT & - & - & 21.1 & - \\
\midrule
% Qwen2.5-0.5B       & 0.5B & -     & - & - & - & Zero-shot & T=0, top\_p=1.0 & 640 & 16.9 & 42.7 & 42.0 & 30.2 \\
% SFT                & 0.5B & GSM8K & SFT & - & - & Zero-shot & T=0, top\_p=1.0 & 640 & 28.6 & 57.6 & 38.1 & 14.2 \\
% Qwen2.5-0.5B+\grpoclip{}          & 0.5B & GSM8K & GRPO & Rule & O & Zero-shot & T=0, top\_p=1.0 & 640 & 54.1 & 77.8 & 66.6 & 43.0 \\
% Qwen2.5-0.5B+\grponoclip{}        & 0.5B & GSM8K & GRPO (iterations=1) & Rule & O &  Zero-shot & T=0, top\_p=1.0 & 640 & 56.0 & 42.6 \\
% Qwen2.5-0.5B+\pppo\cite{chen_sufficiency_2022} & 0.5B & GSM8K & P3O & Rule & O & Zero-shot & T=0, top\_p=1.0 & 640 & 44.5 & 68.5 & 58.3 & 38.9 \\
% Qwen2.5-0.5B+\sapo\cite{gao_soft_2025} & 0.5B & GSM8K & SAPO & Rule & O & Zero-shot & T=0, top\_p=1.0 & 640 & 31.8 & 62.0 & 51.8 & 36.2 \\
% Qwen2.5-0.5B+\topr\cite{roux_tapered_2025} & 0.5B & GSM8K & TOPR & Rule & O & Zero-shot & T=0, top\_p=1.0 & 640 &  \\
% Qwen2.5-0.5B+\ourmethodgr{}   & 0.5B & GSM8K & GR-PSPO & Rule & O & Zero-shot & T=0, top\_p=1.0 & 640 & 53.0 & 41.9 \\
% Qwen2.5-0.5B+\grpoclip{}          & 0.5B & GSM8K+MATH & GRPO & Rule & O & Zero-shot & T=0, top\_p=1.0 & 640 & 49.4 & 76.5 & 62.1 & 34.4\\
Qwen2.5-0.5B+\grponoclip{}        & 0.5B & GSM8K+MATH & GRPO (iterations=1) & Rule & O & Zero-shot & T=0, top\_p=1.0 & 640 & \textbf{54.1} & 36.2\\
% Qwen2.5-0.5B+\pppo\cite{chen_sufficiency_2022} & 0.5B & GSM8K+MATH & P3O & Rule & O & Zero-shot & T=0, top\_p=1.0 & 640 & 50.1 & 75.0 & 62.8 & 36.1 \\
% Qwen2.5-0.5B+\sapo\cite{gao_soft_2025} & 0.5B & GSM8K+MATH & SAPO & Rule & O & Zero-shot & T=0, top\_p=1.0 & 640 & 50.3 & 76.6 & 63.3 & 34.5 \\
% Qwen2.5-0.5B+\topr\cite{roux_tapered_2025} & 0.5B & GSM8K+MATH & TOPR & Rule & O & Zero-shot & T=0, top\_p=1.0 & 640 &  \\
Qwen2.5-0.5B+\ourmethodgr{}   & 0.5B & GSM8K+MATH & GR-PSPO & Rule & O & Zero-shot & T=0, top\_p=1.0 & 640 & 50.9 & \textbf{37.5} \\
\midrule
% Qwen2.5-1.5B       & 1.5B & -     & - & - & - & Zero-shot & T=0, top\_p=1.0 & 640 & 5.2 & 5.0 & 4.0 & 17.6 \\
% Qwen2.5-1.5B+\grpoclip{}          & 1.5B & GSM8K & GRPO & Rule & O & Zero-shot & T=0, top\_p=1.0 & 128 & 37.8 & 70.9 & 58.4 & 14.9  \\
% Qwen2.5-1.5B+\grponoclip{}        & 1.5B & GSM8K & GRPO (iterations=1) & Rule & O & Zero-shot & T=0, top\_p=1.0 & 128 & 57.9 & \textbf{80.4} & \textbf{74.9} & 25.2 \\
% Qwen2.5-1.5B+\ourmethodgr{}   & 1.5B & GSM8K & GR-PSPO & Rule & O & Zero-shot & T=0, top\_p=1.0 & 128 & \textbf{59.4} & 77.7 & 70.3 & \textbf{25.6} \\
% \midrule
% Qwen2.5-1.5B+\grpoclip{}          & 1.5B & GSM8K+MATH & GRPO & Rule & O & Zero-shot & T=0, top\_p=1.0 & 640 & 73.7 & 88.8 & 83.3 & 45.9 \\
\midrule
\multicolumn{11}{l}{\textbf{Mid-to-Large Language Models $>$ 1B parameters.}}\\
\midrule
GPT-2-XL~\citep{brown_language_2020} & 1.5B  & GSM8K+MATH  & N/A & N/A & N/A & Few-Shot & - & 2048 & 15.4 & 6.9 \\
WizardMath-GPT2-XL~\citep{luo_wizardmath_2025} & 1.5B & GSM8K+MATH & PPO & Model & P+O & - & T=0 & 2048 & 58.9 & 25.4 \\
Llama 3 8B+\topr~\cite{roux_tapered_2025} & 8B & GSM8K+MATH & TOPR & Rule & O & Few-Shot & - & 512 (completion tokens) & 75.4 & 22.7 \\
\midrule
Qwen2.5-1.5B+\grponoclip{}        & 1.5B & GSM8K+MATH & GRPO (iterations=1) & Rule & O & Zero-shot & T=0, top\_p=1.0 & 640 & 75.4 & 48.0 \\
Qwen2.5-1.5B+\ourmethodgr{}   & 1.5B & GSM8K+MATH & GR-PSPO & Rule & O & Zero-shot & T=0, top\_p=1.0 & 640 & 73.8 & 50.0 \\
% Qwen2-Math-1.5B       & 1.5B & -     & - & - & - & Zero-shot & T=0, top\_p=1.0 & 640 & 24.7 & 61.8 & 56.3 & 34.8 \\
% Qwen2-Math-1.5B+\grpoclip{}          & 1.5B & GSM8K+MATH & GRPO & Rule & O & Zero-shot & T=0, top\_p=1.0 & 128 &   \\
% Qwen2-Math-1.5B+\grponoclip{}        & 1.5B & GSM8K+MATH & GRPO (iterations=1) & Rule & O & Zero-shot & T=0, top\_p=1.0 & 640 & 78.4 & 92.7 & 84.1 & 58.3 \\
% Qwen2-Math-1.5B+\grpoclip{}        & 1.5B & GSM8K+MATH & GRPO & Rule & O & Zero-shot & T=0, top\_p=1.0 & 640 &  \\
% Qwen2-Math-1.5B+\pppo{}   & 1.5B & GSM8K+MATH & \pppo{} & Rule & O & Zero-shot & T=0, top\_p=1.0 & 640 & 78.8 & 93.0 & 84.7 & 55.8 \\
% Qwen2-Math-1.5B+\sapo{}   & 1.5B & GSM8K+MATH & \sapo{} & Rule & O & Zero-shot & T=0, top\_p=1.0 & 640 & 79.2 & 93.8 &  &  \\
Qwen2-Math-1.5B+\ourmethodgr{}   & 1.5B & GSM8K+MATH & GR-PSPO & Rule & O & Zero-shot & T=0, top\_p=1.0 & 640 & \textbf{79.9 }& \textbf{59.6} \\
% Qwen2-Math-7B       & 7B & -     & - & - & - & Zero-shot & T=0, top\_p=1.0 & 640 & \\
\bottomrule
\end{tabular}
}
\label{tab:comparison}
\end{table*}

\paragraph{Smoothing Alpha Sensitivity}
We ablate $\alpha$ on Qwen2-Math-1.5B (Appendix \ref{app:alpha}). Performance peaks around $\alpha$ = 0.4.

\paragraph{Preliminary larger model results (7B).} 
To further consider scaling behaviour we ran one exploratory seed for a reduced number of tokens on 7B models. The Qwen2.5-7B run shows a similar trend as our smaller-scale experiments, with: \ourmethodgr{} outperforming on MATH. For example, PSPO’s top-1 (true) accuracy is 55.8\% on MATH compared with the base model's 33.9\%, and \grpoclip{}'s 53.6\%. \grponoclip{} has an accuracy of 0\% indicating complete collapse.
These 7B measurements are single-seed and exploratory, so we do not use them to change conclusions drawn from the 3-seed experiments.

\section{Discussion and Limitations}
% Compared with ratio clipping and explicit KL regularisation, \ourmethod{} provides stability without needing to truncate the surrogate objective, and it does so without adding compute or an extra optimisation objective.

Unlike clipping which truncates the surrogate objective and discards gradient information outside the allowed range, \ourmethod{} contracts ratios toward 1 through probability smoothing while preserving non-zero gradients everywhere. This is achieved without additional compute, memory overhead, or extra optimisation objectives, the only change is a linear transformation of the importance ratio.

Our results across three model configurations validate that probability smoothing is most effective when stability is required, particularly when refining a model with existing domain knowledge.

This is shown by performance improvements scaling with domain specialisation. On general base models (Qwen2.5-0.5B/1.5B), GR-PSPO shows advantages primarily on the hardest benchmark (MATH). On the specialised model (Qwen2-Math-1.5B), GR-PSPO outperforms consistently across all tasks, suggesting that PSPO's benefits are clearest when there is existing knowledge to preserve.

Additionally, performance improves with scale. The progression from 0.5B (mixed results) to 1.5B (strong MATH performance) to Math-1.5B (consistent gains) demonstrates that PSPO's advantages compound as model size grows. On MATH, GR-PSPO transitions from -1.1\% compared to clipping (0.5B) to +4.1pp absolute advantage (1.5B).

When fine-tuning very small base models (0.5B), the cost of discarding gradient information may be less impactful. But as models grow and accumulate knowledge, each sample becomes more costly to generate, and preserving learning signal everywhere becomes critical. PSPO's non-zero gradients (Proposition~\ref{prop:pgradient}) provide exactly this property. Unlike clipping's hard boundary, PSPO's behaviour-anchored smoothing provides a soft constraint that scales naturally with how far the policy deviates. Sigmoid-based alternatives such as P3O and SAPO are more vulnerable to saturation at extreme ratios and asymmetric boundaries require additional hyperparameters (SAPO and TOPR each require at least two $\tau$ values; TOPR can require up to four).

Compared to \grponoclip{}, \ourmethodgr{} achieves similar or better performance while enabling multi-epoch training. \grponoclip{} avoids the need for trust regions by using only a single iteration over each batch, ensuring the importance ratio is always 1. This can work well when fresh on-policy samples are cheap, but generating rollouts from LLMs is computationally expensive. When reusing data, the policy diverges from the behaviour policy, making trust region mechanisms essential for stability. Additionally, \citet{zheng_group_2025} and \citet{normaluhr2025grpo} note that larger models often require processing rollout batches in mini-batches to fit in memory, making single-iteration training impractical. PSPO enables stable multi-epoch training without sacrificing gradient information.

Refining a base model with domain knowledge is one of the most relevant for production LLM deployment. In practice, organisations fine-tune capable foundation models that already possess substantial knowledge and capabilities. The goal is improving performance on specific tasks while preserving the model's general competence, which is exactly where our results show PSPO excels. 

%\section{Limitations}
% While our results demonstrate the effectiveness of \ourmethodgr{} on mathematical reasoning tasks, there are limitations in our experimental approach. 

% Our experimental evaluation is only on mathematical reasoning, where there are binary, objective reward signals. The effectiveness of probability smoothing in domains with more subjective or continuous rewards remains unexplored, and should be considered in future work. This will also provide more insight into the sensitivity of $\alpha$ across domains.

% The scale of our experiments is limited to models under 2B parameters. In practice, larger models are normally deployed. Although we demonstrate \ourmethodgr{} performs across two model sizes, future work should consider larger model sizes, as well as different architectures and tokeniser uses.

% Finally, \ourmethodgr{} achieves similar quantitative performance to \grponoclip{}, and we have noted that literature\citep{zheng_group_2025, normaluhr2025grpo} both have claimed GRPO struggles with larger models and more complex architectures. Empirically comparing \ourmethodgr{} against GRPO in these settings will allow a fuller characterisation of where our method provides practical advantages.

\section{Conclusions}
We introduced Probability Smoothing Policy Optimisation (\ourmethod{}) as a gradient-preserving alternative to ratio clipping in reinforcement learning for large language models, which mixes the current policy with the behaviour policy before forming the importance ratio. This blend results in a behaviour-anchored \emph{soft trust region} that linearly contracts ratios around $r{=}1$, shrinks TV/KL divergence bounds (Lemma \ref{lem:tv}, Corollary \ref{cor:kl}), and preserves non-zero gradients everywhere (Proposition \ref{prop:ratio}). 

We empirically evaluate our method by implementing \ourmethod{} within GRPO as \ourmethodgr{}, our method demonstrates clear advantages in settings where stability matters most. On Qwen2-Math-1.5B, a model with prior supervised mathematical training, \ourmethodgr{} outperforms baselines across all four benchmarks; giving results similar to Llama 3 8B+TOPR~\cite{roux_tapered_2025}. On base Qwen2.5-1.5B, \ourmethodgr{} achieves the strongest MATH performance, exceeding \grpoclip{} by 4.1pp. These results support our theoretical motivation that when there is existing capability to preserve, \ourmethod{}'s behaviour-anchored smoothing prevents the destabilising updates more effectively than clipping.

% Beyond accuracy, GR-PSPO consistently produces higher-quality responses as measured by our LLM-as-judge evaluation, with improvements in constraint adherence, logical coherence, and clarity. This suggests that the overconfidence regularisation induced by probability smoothing helps maintain structured output patterns.

Additionally, \ourmethod{} is a drop-in replacement for clipping. \ourmethod{} requires tuning of only one hyperparameter ($\alpha$, replacing $\varepsilon$), and adds no computational or memory overhead. These properties make it particularly suited to settings where data reuse is necessary and preserving existing model capabilities is critical, the common case in post-SFT reinforcement learning pipelines.

% \section*{Acknowledgements}
% We thank Stephen Pasteris for helpful discussions and feedback on the mathematical aspects of this work. % As he briefly checked over the proofs for me and gave me some advice on notation consistency
% This work was supported by UK Research and Innovation [EP/S024298/1]. % the cdt required acknowledgement
% The authors acknowledge the use of the IRIDIS High Performance Computing Facility, and associated support services at the University of Southampton, in the completion of this work.' % the iridis acknowledgement

\section*{Impact Statement}
% This work presents an alternative method to provide stability to training for fine-tuning large language models. By improving stability during post-training, PSPO may reduce wasted computation from failed training runs,. The method itself is agnostic to the reward signal used, meaning its societal impact depends on downstream applications. We do not foresee unique risks beyond those inherent to LLM fine-tuning more broadly.

This paper presents work whose goal is to advance the field of Machine Learning. There are many potential societal consequences of our work, none which we feel must be specifically highlighted here.


\bibliography{main}
\bibliographystyle{icml2026/icml2026}

%%%%%%%%%%%%%%%%%%%%%%%%%%%%%%%%%%%%%%%%%%%%%%%%%%%%%%%%%%%%%%%%%%%%%%%%%%%%%%%
%%%%%%%%%%%%%%%%%%%%%%%%%%%%%%%%%%%%%%%%%%%%%%%%%%%%%%%%%%%%%%%%%%%%%%%%%%%%%%%
% APPENDIX
%%%%%%%%%%%%%%%%%%%%%%%%%%%%%%%%%%%%%%%%%%%%%%%%%%%%%%%%%%%%%%%%%%%%%%%%%%%%%%%
%%%%%%%%%%%%%%%%%%%%%%%%%%%%%%%%%%%%%%%%%%%%%%%%%%%%%%%%%%%%%%%%%%%%%%%%%%%%%%%
\newpage
\appendix
\onecolumn
\section{Training Details}
\label{app:tuning}
The training for each method was run on 1 GPU. We used python version 3.10.18, and trl=0.25.1. We save checkpoint models every 1000 steps, and return the best model checkpoint at the end of training.

We used Nvidia GPUs (A100) and AMD GPUs (MI300X) for training, using only 1 GPU at a time. It took approximately 18 hours per training run for the 150M tokens, and around 10 hours for the 75M token training runs.

\subsection{Tuning Method}
Each method was independently tuned using a grid search over learning rate and any parameter related to the trust region (i.e. clipping range, smoothing alpha value, tau values).

We ran the initial grid search using seed 42, for 500 max global steps. % (hyperband argument for why it can be shorter?). 
The top 5 configurations were then run again on seeds 43 and 44. The best average validation reward from this was selected for the full training runs. These seeds are different to those used during training (seeds 0, 1 and 2 were used in training).

 We used the learning rates $= [5e-07, 1e-06, 5e-06]$, and the method specific values used are shown in Table \ref{tab:parameters_grid}.

\begin{table}[ht] 
\centering 
\small 

\caption{Hyperparameter grid search values for method specific parameters.}

\resizebox{\textwidth}{!}{% 
\begin{tabular}{p{2cm}p{3.5cm}lp{5cm}} 
\toprule 
Method & Trust Region Parameter & Grid Search Values & Notes \\ 
\midrule
GRPO-clip & clipping range, $\varepsilon$ & [0.1, 0.2, 0.3, 0.4] & \\ 
GRPO-noclip & N/A & N/A & \\ 
P3O & temperature, $\tau$ & [1.0, 2.0, 4.0, 10.0] & \cite{chen_sufficiency_2023} suggested $\tau = 4.0$ \\ 
\multirow{2}{=}{SAPO} & temperature for positive advantages, $\tau_{positive}$ & [1.0, 2.0, 3.0] & \multirow{2}{=}{\cite{gao_soft_2025} suggested that $\tau_{negative} > \tau_{positive}$ gave the most stable training, so we only considered configurations where $\tau_{negative} \geq \tau_{positive}$.}\\ 
 & temperature for negative advantages, $\tau_{negative}$ & {[}1.0, 2.0, 3.0{]} & \\[3.5em]  % <-- add space here
\multirow{2}{=}{TOPR} & truncation limit for positive advantages, $b^+$ & {[}1.0, 2.0, 10.0{]} & \multirow{2}{=}{\cite{roux_tapered_2025} suggested $a^- = 0$ and $a^+ = 1$; we searched across $b^+$ and $b^-$ values (paper uses $b^+ = b^- = 1$).} \\ 
 & truncation limit for negative advantages, $b^-$ & [1.0, 2.0, 10.0] & \\[2em]  % <-- and here
GR-PSPO & smoothing alpha, $\alpha$ & [0.1, 0.2, 0.3, 0.4] & \\ 
\bottomrule
\end{tabular} 
} 
\label{tab:parameters_grid}
\end{table}

\subsection{Hyperparameter Values} % for Mathematics Reasoning Task}
\subsubsection{Qwen2.5-0.5B and Qwen2.5-1.5B}
For the Qwen2.5-0.5B and -1.5B models, we used the shared parameters shown in Table \ref{tab:shared}.
\begin{table}[ht]
\centering
\caption{Shared training hyperparameters for the 0.5B and 1.5B experiments.}
\label{tab:shared}
\begin{tabular}{ll}
\toprule
Parameter & Value \\
\midrule
Effective batch size & 128 (16 per device $\times$ 8 accumulation steps) \\
Generations per prompt & 4 \\
KL penalty ($\beta$) & 0.0 \\
Max completion length & 128 tokens \\
Sampling temperature & 0.6 \\
Top-$p$ & 0.85 \\
Warmup steps & 125 \\
Precision & bfloat16 \\
\bottomrule
\end{tabular}
\end{table}

The method specific parameters are shown in Table \ref{tab:hyperparams05}.

\begin{table}[ht]
\centering
\caption{Hyperparameters for each trust region method for Qwen2.5-0.5B and -1.5B. All methods use the DAPO loss unless otherwise noted.}
\label{tab:hyperparams05}
\begin{tabular}{lccccc}
\toprule
Method & Iterations & Learning Rate & \multicolumn{3}{c}{Trust Region Parameters} \\
\midrule
GRPO-noclip & 1 & 1e-6 & \multicolumn{3}{c}{---} \\
GRPO-clip & 2 & 1e-6 & \multicolumn{3}{c}{$\varepsilon = 0.4$} \\
GR-PSPO & 2 & 1e-6 & \multicolumn{3}{c}{$\alpha = 0.1$} \\
P3O & 2 & 1e-6 & \multicolumn{3}{c}{$\tau = 4.0$} \\
SAPO\textsuperscript{\dag} & 2 & 1e-6 & \multicolumn{3}{c}{$\tau_{+} = 1.0$, $\tau_{-} = 3.0$} \\
TOPR\textsuperscript{\dag} & 2 & 1e-6 & \multicolumn{3}{c}{$b^{+} = 1.0$, $b^{-} = 1.0$} \\
\bottomrule
\end{tabular}

\smallskip
{\footnotesize \textsuperscript{\dag}Uses GRPO-style loss to match existing implementations.}
\end{table}

\subsubsection{Qwen2-Math-1.5B}
For the Qwen2-Math-1.5B model, we used the shared parameters shown in Table \ref{tab:shared}.

The method specific parameters are shown in Table \ref{tab:hyperparams15}.

\begin{table}[ht]
\centering
\caption{Hyperparameters for each trust region method for Qwen2-Math-1.5B. All methods use the DAPO loss unless otherwise noted.}
\label{tab:hyperparams15}
\begin{tabular}{lccccc}
\toprule
Method & Iterations & Learning Rate & \multicolumn{3}{c}{Trust Region Parameters} \\
\midrule
GRPO-noclip & 1 & 5e-6 & \multicolumn{3}{c}{---} \\
GRPO-clip & 2 & 5e-6 & \multicolumn{3}{c}{$\varepsilon = 0.2$} \\
GR-PSPO & 2 & 1e-6 & \multicolumn{3}{c}{$\alpha = 0.4$} \\
P3O & 2 & 1e-6 & \multicolumn{3}{c}{$\tau = 4.0$} \\
SAPO\textsuperscript{\dag} & 2 & 1e-6 & \multicolumn{3}{c}{$\tau_{+} = 1.0$, $\tau_{-} = 3.0$} \\
TOPR\textsuperscript{\dag} & 2 & 1e-6 & \multicolumn{3}{c}{$b^{+} = 1.0$, $b^{-} = 1.0$} \\
\bottomrule
\end{tabular}

\smallskip
{\footnotesize \textsuperscript{\dag}Uses GRPO-style loss to match existing implementations.}
\end{table}

% \subsection{Hyperparameter Values for Summarisation Task}
% \subsubsection{0.5B}

% \subsubsection{1.5B}


% You can have as much text here as you want. The main body must be at most $8$
% pages long. For the final version, one more page can be added. If you want, you
% can use an appendix like this one.

% The $\mathtt{\backslash onecolumn}$ command above can be kept in place if you
% prefer a one-column appendix, or can be removed if you prefer a two-column
% appendix.  Apart from this possible change, the style (font size, spacing,
% margins, page numbering, etc.) should be kept the same as the main body.
% %%%%%%%%%%%%%%%%%%%%%%%%%%%%%%%%%%%%%%%%%%%%%%%%%%%%%%%%%%%%%%%%%%%%%%%%%%%%%%%
% %%%%%%%%%%%%%%%%%%%%%%%%%%%%%%%%%%%%%%%%%%%%%%%%%%%%%%%%%%%%%%%%%%%%%%%%%%%%%%%

\section{Sensitivity to Smoothing Alpha}
\label{app:alpha}
We ablate the smoothing parameter $\alpha$ on Qwen2-Math-1.5B with a learning rate of 1e-6 for 250 global steps. Performance peaks around $\alpha$=0.4, with degradation at both extremes, which could support our theory that insufficient smoothing ($\alpha$=0.1) provides weak regularisation, while excessive smoothing ($\alpha$=0.5) over-constrains updates. Table \ref{tab:alpha} shows the mean validation reward at the end of the 250 steps across different values of $\alpha$.

\begin{table}[ht]

\caption{Small ablation across smoothing alpha, $\alpha$, values on Qwen2-Math-1.5B.}
\centering
\begin{tabular}{cc}
\toprule
$\alpha$ & Mean Validation Reward \\
\midrule
0.1      & 0.170                  \\
0.2      & 0.276                  \\
0.4      & 0.319                  \\
0.5      & 0.240              \\
\bottomrule
\end{tabular}
\label{tab:alpha}
\end{table}

\end{document}